\documentclass[journal]{IEEEtran}

\ifCLASSINFOpdf
\else
   \usepackage[dvips]{graphicx}
\fi
\usepackage{amsthm}
\usepackage{url}
\usepackage{cite}

\usepackage{amsmath,amssymb,amsfonts}
\usepackage{algorithmic}
\usepackage{graphicx}
\usepackage{textcomp}
\usepackage{xcolor}
\usepackage{enumerate}  
\DeclareMathOperator*{\argmin}{arg\,min}

\newcommand{\bLambda}{\boldsymbol{\Lambda}}
\newcommand{\bSigma}{\mathbf{\Sigma}}

\newcommand{\sign}{\text{sign}}

\newcommand{\bA}{\mathbf{A}}
\newcommand{\bB}{\mathbf{B}}

\newcommand{\bG}{\mathbf{G}}

\newcommand{\bI}{\mathbf{I}}

\newcommand{\bL}{\mathbf{L}}

\newcommand{\bW}{\mathbf{W}}
\newcommand{\bX}{\mathbf{X}}
\newcommand{\bY}{\mathbf{Y}}
\newcommand{\bZ}{\mathbf{Z}}

\newcommand{\bb}{\mathbf{b}}

\newcommand{\bg}{\mathbf{g}}

\newcommand{\bw}{\mathbf{w}}
\newcommand{\bx}{\mathbf{x}}
\newcommand{\by}{\mathbf{y}}
\newcommand{\bz}{\mathbf{z}}

\newcommand{\hbW}{\hat{\mathbf{W}}}

\newcommand{\bbE}{\mathbb{E}}
\newcommand{\bbR}{\mathbb{R}}

\newcommand{\bbN}{\mathbb{N}}

\newcommand{\calB}{\mathcal{B}}

\newcommand{\calF}{\mathcal{F}}

\newcommand{\calL}{\mathcal{L}}
\newcommand{\calN}{\mathcal{N}}
\newcommand{\calM}{\mathcal{M}}

\newcommand{\calS}{\mathcal{S}}

\newcommand{\bzero}{\mathbf{0}}

\newcommand{\tr}{\text{Tr}}

\newtheorem{theorem}{Theorem}
\newtheorem{Proposition}{Proposition}
\newtheorem{fact}{Fact}
\newtheorem{remark}{Remark}
\newtheorem{lemma}{Lemma}
\newtheorem{corollary}{Corollary}
\newtheorem{definition}{Definition}
\newtheorem{assumptions}{Assumptions}

\begin{document}

\title{Dual Space Preconditioning for Gradient Descent in the Overparameterized Regime}

\author{ Reza Ghane, \IEEEmembership{Graduate Student Member, IEEE}, Danil Akhtiamov, \IEEEmembership{Graduate Student Member, IEEE}, Babak Hassibi, \IEEEmembership{Member, IEEE}  \thanks{Reza Ghane and Danil Akhtiamov contributed equally.} \thanks{Reza Ghane and Babak Hassibi are with the Electrical Engineering Depart-
ment, California Institute of Technology, Pasadena, CA 91125 USA (e-mail:
rghanekh@caltech.edu; hassibi@caltech.edu).\\
Danil Akhtiamov is with the Computing + Mathematical Sciences Depart-
ment, California Institute of Technology, Pasadena, CA 91125 USA (e-mail:
dakhtiam@caltech.edu).}}
\maketitle

\markboth{}
{Shell \MakeLowercase{\textit{et al.}}: Bare Demo of IEEEtran.cls for IEEE Journals}

\begin{abstract}

In this work, we study the convergence properties of the Dual Space Preconditioned Gradient Descent, encompassing optimizers such as Normalized Gradient Descent and Gradient Clipping. We consider preconditioners of the form $\nabla K$, where $K: \bbR^{d \times k} \to \bbR$ is convex and apply $\nabla K(\cdot)$ to train an over-parameterized linear model with a convex loss of the form $\ell(\bX \bW - \bY)$, for weights $\bW \in \bbR^{d \times k}$, labels $\bY \in \bbR^{n \times k}$ and data $\bX \in \bbR^{n \times d}$. Under the aforementioned assumptions, we prove that the iterates of the full-batch preconditioned gradient descent converge at an exponential rate to a point $\bW_{\infty} \in \bbR^{d \times k}$ satisfying $\bX\bW_{\infty} = \bY$. 

We also study the implicit bias of Dual Space Preconditioned Gradient Descent. First, we demonstrate analytically and empirically that, for general $K(\cdot)$, $\bW_\infty$ depends on the chosen constant step size, hindering a precise characterization of the implicit bias. We also provide an approximate implicit bias property for general preconditioners, namely, $\|\bW_0 - \bW_{\infty}\|_F \le c \|\bW_0 - \bW_{\text{GD}, \infty}\|_F$ for a constant $c>0$ and $\bW_{\text{GD}, \infty}$ denoting the convergence point of GD initialized at $\bW_0$. Furthermore, for preconditioners of the form $K(\bG) = h(\|\bG\|_F)$,  known as {\it isotropic preconditioners}, and for the stochastic variation of the algorithm with arbitrary batch-size, we prove linear convergence to $\bW_{\text{GD}, \infty}$. Finally, in the experiments, we demonstrate faster convergence on a nonlinear model obtained using the smoothed matrix elastic-net as a preconditioner.



\end{abstract}

\IEEEpeerreviewmaketitle

\section{Introduction}

The empirical success of stochastic gradient descent and its adaptive extensions (e.g., Adam \cite{kingma2014adam}, Adagrad \cite{duchi2011adaptive}) in neural network training has motivated extensive recent research aimed at enhancing gradient-based optimization methods. Methods such as SignSGD \cite{bernstein2018signsgd}, Gradient Clipping \cite{zhang2019gradient}, Spectral GD \cite{bernstein2024old}, and Adam share a common characteristic: they employ a nonlinear function of the loss gradient in the update rule. To investigate this family of algorithms more systematically,\cite{maddison2021dual} proposed considering {\it Dual Space Preconditioned GD}, defined as 
\begin{align*}
    \bw_t =  \bw_{t-1} - \eta \nabla K\left(\nabla\calL(\bw_{t-1})\right),
\end{align*}
where $K$ is a convex function, $\calL$ is the empirical loss and $\bw_t \in \bbR^d$ denotes the weights at time step $t$. 
Despite an existing body of work \cite{maddison2021dual, laude2023dualities, laude2025anisotropic, oikonomidis2025nonlinearly}, our theoretical understanding of the convergence properties of the Dual Space Preconditioned GD in the overparameterized regime remains limited. 

In this paper, we consider the following optimization $\min_{\bW \in \bbR^{d\times k}} \calL(\bW)$ problem and apply the full-batch Dual Space Preconditioned GD to minimize the loss:
\begin{align}\label{alg: precond}
    \bW_t =  \bW_{t-1} - \eta \nabla K\left(\nabla\calL(\bW_{t-1})\right)
\end{align}
Here, $\bW \in \bbR^{d \times k}$ is matrix-valued. Furthermore, often, for the empirical loss of the form $ \min_{\bW \in \bbR^{d\times k}} \frac{1}{n} \sum_{i=1}^n \calL_i(\bW)$, where $\calL_i$ is the loss associated with the $i$th data point, we use the stochastic Dual Space Preconditioned GD with batch size $B$:
\begin{align}\label{alg: stoch}
    \bW_t =  \bW_{t-1} - \eta \nabla K\left(\nabla\calL^{(B)}_t(\bW_{t-1})\right)
\end{align}
Here, $\calL^{(B)}_t(\cdot):= \frac{1}{B} \sum_{i = 1}^B \calL_{t_i}(\cdot)$ where $t_i \in [n]$ are selected uniformly without replacement and independent of previous steps.
Another contribution of the current manuscript is that we incorporate the matrix structure of $\bW$ in our analysis, allowing for matrix preconditioners, whereas in previous works \cite{maddison2021dual, laude2025anisotropic, oikonomidis2025nonlinearly}, only the vector structure of $\bW \in \bbR^{p}$ was taken into account.  We motivate introducing matrix structure by the recent success of matrix preconditioners, such as Muon, SOAP, and Shampoo \cite{jordan2024muon, vyas2024soap, gupta2018shampoo}. 
We focus on convex loss functions of the form $\calL(\bW) := \ell(\bX\bW-\bY)$ where $\bX \in \bbR^{n \times d}$ is the data matrix, $\bW \in \bbR^{d \times k}$ is the weight matrix being optimized and $\bY \in \bbR^{n \times k}$ is the label matrix. In this work, we consider the regime where $n<d$ and thus $\ell(\bX\bW - \bY)$ {\it is not strictly convex and does not have a unique minimizer}. It should be noted that most works in the literature assume strict convexity, making our proof of convergence for Dual Space Preconditioned GD novel in this setting.  

Since $\bX\bW = \bY$ has many solutions due to the over-parametrization,  it is natural to further inquire into the characterizing properties of $\bW_{\infty}$ among this set of solutions, usually known as {\it implicit bias} in the literature.  As an example of an algorithm displaying implicit regularization, it has been shown \cite{akhtiamov2026implicitbiasconvergencematrix} that for the loss $\ell(\bX\bW-\bY)$ mirror descent (MD) and stochastic mirror descent (SMD) with a potential $\psi(\cdot)$, find the interpolating point closest in Bregman Divergence to the initialization, i.e., $\bX \bW = \bY$. In other words,
\begin{align*}
    \min_{\bw} D_{\psi}(\bW, \bW_0) \\
    s.t \quad \bX \bW = \bY
\end{align*}
Different choices of $\psi$ lead to interpolating weights that differ in properties and generalization error \cite{gunasekar2018characterizing, azizan2018stochastic,akhtiamov2026implicitbiasconvergencematrix}. 

As outlined in the abstract, the remainder of our paper is devoted to analyzing the implicit bias of Dual Space Preconditioned GD, with particular attention to the case of isotropic preconditioners. In addition, we establish an exponential convergence rate of the weights for isotropic preconditioners.

\subsection{Related Work}

\textbf{Dual Space preconditioning: } Maddison et al. \cite{maddison2021dual} coined the term "Dual Preconditioning". For convex loss functions $f$ with a unique minimizer, \cite{maddison2021dual} provides a convergence rate for the loss. In contrast to the setting of \cite{maddison2021dual}, we operate in the overparameterized regime where the loss function does not have a unique minimizer. Moreover, the Dual Space Preconditioned GD belongs to a family of algorithms called Lion-$\mathcal{K}$ introduced in \cite{chen2023lion}.
Several works extend the results of \cite{maddison2021dual} to nonconvex loss functions \cite{laude2025anisotropic, bodard2025escaping, oikonomidis2025nonlinearly}. Furthermore, \cite{oikonomidis2025nonlinearlymom} investigated Dual Space Preconditioned GD \eqref{alg: precond} with momentum and a batch size of one. Nonetheless, none of the aforementioned works investigate the convergence of the \textit{weights} as well as the implicit bias.

\textbf{Implicit bias in the overparameterized regime:} The work \cite{wilson2017marginal} was among the first to motivate the study of algorithms in the overparameterized regime, in which, despite achieving a zero training error, the algorithm might generalize poorly. They construct an example on which AdaGrad generalizes worse than SGD.

The convergence and the implicit bias of mirror descent \cite{gunasekar2018characterizing} and its stochastic variants \cite{azizan2018stochastic, Varma2025Exponential, akhtiamov2024regularized} have attracted attention. Furthermore, for the convergence of mirror descent on nonlinear models, \cite{azizan2021stochastic} has characterized an approximate implicit regularization property.

Another direction in implicit bias research deals with the convergence of the direction of a classifier. For GD and SGD, \cite{soudry2018implicit, nacson2019stochastic} demonstrate that the iterates of GD and SGD for minimizing the logistic loss on linearly separable data converge in direction to $\ell_2$ max-margin classifier. \cite{qian2019implicit} shows the convergence in direction of the Adagrad iterations to the solution of a hard-constrained weighted quadratic optimization problem. Furthermore, \cite{zhang2024implicit} proved that the iterates of Adam would converge in direction to the max $\ell_\infty$-margin classifier, contingent on uniqueness.  \cite{fan2025implicit} established that, for the matrix normalized steepest descent, by considering the cross-entropy loss, the normalized margin gap of the iterates converges to the Schatten-$p$ max-margin value.

\subsection{Contributions}
\begin{enumerate}
    \item We prove the convergence and provide a convergence rate for the weights generated by algorithm \eqref{alg: precond} for convex losses in the overparameterized regime and provide an upper bound on the distance of the convergence point of \eqref{alg: precond} from the final GD solution which depends on the distance of the initialization point from the latter. We apply our results to two particular examples, a simplified Adam-like update and a smoothed matrix elastic-net as the preconditioner.
    \item For the isotropic preconditioners, we show that the iterates of the stochastic algorithm \eqref{alg: stoch} converge to the interpolating solution in minimum Frobenius-norm distance from the initialization. In other words, \eqref{alg: stoch} for the isotropic preconditioner has the same implicit bias as SGD \cite{gunasekar2018characterizing, azizan2018stochastic}. Using these results, we demonstrate the convergence of normalized gradient descent and Stochastic gradient clipping \cite{zhang2019gradient}
    \item We provide an analytical example of a nonisotropic preconditioner in which the point of convergence depends on the step size. Furthermore, for the case of simplified Adam-like update, numerical evidence is given.
\end{enumerate}

\section{Notations and Definitions}

We use bold-faced letters to denote matrices and vectors. $n \ge 1$ denotes the number of data points in the data matrix $\bX \in \bbR^{n \times d}$ and $d > n$ denotes the number of features. Moreover, $\bY \in \bbR^{n \times k}$ denotes the labels. We use $\sigma_i(\bA)$ to denote the $i$th largest singular value of $\bA$. 
\begin{definition}
    We have the following notations:
    \begin{itemize}
    \item We denote the set $range(\bX^T)^k$ by 
\begin{align} \label{def: span_sub}
    \calS := \{\bX^T \bL: \bL \in \bbR^{n\times k}\}
\end{align}
    \item  Define the interpolating manifold $\calM$ as
\begin{align} \label{def: manif}
    \calM := \{\bW \in \bbR^{d\times k}: \bX \bW = \bY\}
\end{align}
    \item For$\rho > 0$, let $\calB_{\rho} := \{\bZ \in \bbR^{d \times k} : \|\bZ\|_F \le \rho \}$ .
\end{itemize}
\end{definition}

\begin{definition}
[Strong Convexity]
     $f: \bbR^{d \times k}\rightarrow \bbR$ is $\mu$-strongly convex if 
    \begin{align*}
        f(\bW_2) &\ge f(\bW_1) + \tr\Bigl(\nabla f(\bW_1)^T (\bW_2 - \bW_1) \Bigr)  \\ & +\frac{\mu}{2} \|\bW_1 - \bW_2 \|_F^2
    \end{align*}
\end{definition}

\begin{definition}\label{def: fenchel}
(Fenchel Conjugate)
    We define $f^{\ast}: \bbR^{d \times k}\rightarrow \bbR$ to be the Fenchel conjugate of $f: \bbR^{d \times k}\rightarrow \bbR$ if
    \begin{align*}
        f^{\ast}(\bZ) = \max_{\bZ' \in \bbR^{d \times k}} \tr( \bZ^T \bZ') - f(\bZ')
    \end{align*}
\end{definition}
\begin{definition}\label{def: lip_grad}
    (Lipschitz Gradient)
    A differentiable function $f: \bbR^{d \times k}\rightarrow \bbR$ has $L$-Lipschitz gradient if for $\bA, \bB \in \bbR^{d \times k}$:
    \begin{align*}
        \|\nabla f(\bA) - \nabla f(\bB) \|_F \le L \|\bA - \bB\|_F
    \end{align*}
\end{definition}
We define $K: \bbR^{d \times k} \to \bbR$ to be a convex function whose gradient defines the preconditioner $\nabla K: \bbR^{d \times k} \to \bbR^{d \times k}$ and $\calL(\bW)$ to be the training loss achieved at $\bW$.

\begin{definition}[Bregman Divergence]\label{def: breg}
    For any $\bA, \bB \in \bbR^{d\times k}$ and $f:\bbR^{d \times k} \rightarrow \bbR$, the standard Bregman Divergence is defined as 
    \begin{align*}
        D_{f}(\bA,\bB) :=&  f(\bA) - f(\bB) - \tr{\left(\nabla f(\bB)^T(\bA - \bB)\right)}
    \end{align*}
\end{definition}



In order to state the fundamental identity used for the analysis of the convergence of algorithm \eqref{alg: precond}, we need the following preliminary result.

 \begin{lemma}[Adjusted Bregman Divergence]\label{lem: gen_breg}
    For any $\bA, \bB \in \bbR^{d\times k}$ and $f:\bbR^{d \times k} \rightarrow \bbR$, the Adjusted Bregman Divergence is defined as 
    \begin{align*}
        \tilde{D}_{f}(\bA,\bB) :=& f^*(\nabla f(\bA)) - f^*(\nabla f(\bB)) \\
    &- \tr{\left(\bB^T(\nabla f(\bA) -\nabla f(\bB))\right)}
    \end{align*}
    It equals $D_f(\bB, \bA)$ by Fenchel duality. \cite{bauschke1997legendre}
\end{lemma}

The first set of assumptions is used to derive a fundamental identity as in Proposition \ref{prop: fund_id}, to establish the convergence, and to show the boundedness of the loss gradients in \eqref{alg: precond}.

\begin{assumptions} \label{ass: main_gen}
    \begin{enumerate}
        \item The dual reference function $K: \bbR^{d\times k} \rightarrow \bbR$ is differentiable and convex. Furthermore $K(\bA) > 0$ for every  $\bA \neq \bzero$, and $K(\bzero) = 0, \nabla K(\bzero) = \bzero$ and $K$ is coercive.
        \item The loss function $\calL : \bbR^ {d \times k} \rightarrow \bbR$ is convex and there exists $\bW_{\ast}$ such that $\nabla \calL(\bW_{\ast}) = 0$.  
   \item $\calL^\ast - \eta K$ is convex on $\calS$, where $\calL^\ast$ is the Fenchel conjugate of $\calL$ as defined in Definition \ref{def: fenchel}.
    \end{enumerate}
\end{assumptions}

The second set of assumptions is leveraged to obtain the convergence rate of \eqref{alg: stoch} and also a bound on the distance between the convergence point of \eqref{alg: precond} from that of GD.

\begin{assumptions}\label{ass: main_erm}
    \begin{enumerate}   
        \item The smallest singular value of $\bX \bX^T \in \bbR^{n \times n}$ is nonzero, i.e., $\sigma_{n}(\bX \bX^T)  > 0$.
        \item $K(\cdot)$ has $L_K$-Lipschitz gradient as in Definition \ref{def: lip_grad}.
        \item The loss function $\calL(\bW)$ can be written as $\ell(\bX\bW - \bY)$ where $\ell: \bbR^{n \times k} \rightarrow \bbR$ is convex. In addition, assume that $\ell$ is separable: $$\calL(\bW) =  \frac{1}{n}
         \sum_{i,j=1}^{n,k}\ell_{i,j}  \left(\bx_i^T\bW^{(j)}-\bY_{ij}\right)$$
        Each $\ell_{i,j}$ satisfies $\ell_{i,j}(0) = 0, \ell'_{i,j}(0) = 0 $, is $\mu$-strongly convex for a $\mu > 0$ and has $M$-Lipschitz gradient as in Definition \ref{def: lip_grad} for $M>0$.    
    \end{enumerate}
\end{assumptions}

\begin{fact}\label{fact: gd_implicit}
    [Implicit Bias of SGD \cite{gunasekar2018characterizing, azizan2018stochastic, akhtiamov2024regularized}] For $\eta < \frac{1}{M \sigma_1(\bX \bX^T)}$, the iterations of \eqref{alg: precond} for $K(\bZ) = \frac{1}{2}\|\bZ\|_F^2$ to minimize $\ell(\bX \bW - \bY)$ converge to
    \begin{align}\label{opt: GD}
    \bW_{\text{GD},\infty} := \argmin_{\bW \in \bbR^{d\times k}} \|\bW - \bW_0\|_F^2  \quad           s.t. \quad \bX \bW = \bY 
    \end{align}
\end{fact}

\section{Main Results}

The key to the proof of convergence is the following fundamental identity for the preconditioned gradient descent:

\begin{Proposition}\label{prop: fund_id}
    For $K$ and $\calL$ satisfying Assumptions 1.1 and 1.2, and for any $\bW \in \bbR^{d\times k}$ from Definition \ref{def: breg} and $\{\bW_t\}_{t=1}$ generated according to \eqref{alg: precond} for $B= n$:
    \begin{align}\label{eq: fund_id}
     \tilde{D}_{\calL}(\bW,\bW_{t-1}) &=\tilde{D}_{\calL}(\bW,\bW_t) + \eta K(\nabla \calL(\bW_{t}))  \nonumber  \\
     &- \eta K(\nabla \calL(\bW))  +\tilde{D}_{\calL}(\bW_t,\bW_{t-1}) \nonumber \\ 
     &- \eta D_{K}\Bigl(\nabla \calL(\bW_t),\nabla \calL(\bW_{t-1})\Bigr)  \nonumber \\
     &+ \eta D_{K}\Bigl(\nabla \calL(\bW),\nabla \calL(\bW_{t-1})\Bigr)
\end{align}
\end{Proposition}

Note that Proposition \ref{prop: fund_id} extends the Descent Lemma (Lemma 3.8) in \cite{maddison2021dual} from an inequality to an equality. Identities of this kind are used to prove the convergence of stochastic mirror descent and to derive the convergence rate \cite{azizan2018stochastic, Varma2025Exponential}.

\subsection{General Full-Batch Dual Space Preconditioner}

In this section, we focus on the algorithm presented in \eqref{alg: precond}.


\begin{theorem} \label{thm: conv_gen}
    (General Preconditioners: Full-Batch) Consider the iterates of \eqref{alg: precond}. Let $\bW_{\text{GD},\infty}$ as in \eqref{opt: GD}.
    
    \begin{enumerate}[I)]
        \item Under Assumptions \ref{ass: main_gen} and \ref{ass: main_erm}.1, $ \|\nabla \calL(\bW_t)\|_F \rightarrow 0 $ as $t \rightarrow \infty$. Subsequently,
        \begin{align*}
            r:=\sup_{t\ge 1} \|\nabla \calL(\bW_t)\|_F < \infty
        \end{align*}
        \item Under Assumptions \ref{ass: main_gen} and \ref{ass: main_erm}, if there exists an $\alpha > 0$ such that $(K - \alpha \calL^\ast)(\bZ)$ is convex on $\calB_r \cap \calS$,  there exists $\bW_{\infty} \in \calM$ in \eqref{def: manif} such that $\bW_t \rightarrow \bW_{\infty}$ as $t \rightarrow \infty$. Moreover, for $t\ge 1$,
        \begin{align*}
            \|\bW_{\infty} - \bW_t \|^2_F \le \eta^2 \frac{2L_K^2\tilde{M} \calL(\bW_0) }{(1- \sqrt{1 - \alpha \eta})^2} (1-\alpha \eta)^{t}
        \end{align*}
        
        If $\nabla K(\nabla L(\bW_t)) \in \calS$ for all $t$, then $\bW_{\infty} = \bW_{\text{GD},\infty}$.
        
        \item Let $\tilde{M}:=  \frac{M \sigma_1(\bX\bX^T)}{n}$ and $ \tilde{\mu} :=  \frac{\mu \sigma_n(\bX\bX^T)}{n} $. Then
        \begin{align*}
             & \|\bW_{\text{GD},\infty} - \bW_{\infty}\|_{F} \le  \tilde{M} \eta  \|\bW_0 -\bW_{\text{GD},\infty} \|_{F} \\
    &\cdot \Bigl( \frac{\sqrt{2} }{1-\sqrt{1-\eta \tilde{\mu}}} + \frac{L_K + 2}{1-\sqrt{1-\alpha \eta}}\Bigr)
        \end{align*}
       And
        \begin{align*}
            &\|\bW_0 - \bW_{\infty}\|_F \le    \|\bW_0 -\bW_{\text{GD},\infty} \|_F \\
            &\cdot \Bigl(1+ \frac{\sqrt{2} \tilde{M} \eta }{1-\sqrt{1-\eta\tilde{\mu} }} + \frac{\tilde{M} \eta (L_K+2)}{1-\sqrt{1-\alpha \eta}}\Bigr)
        \end{align*}

    \end{enumerate}
\end{theorem}

\begin{corollary}
    If $\nabla K(\bZ) = \bZ \tilde{K}(\bZ)$ for some matrix-valued function $\tilde{K}$ where $K$ satisfies Assumptions \ref{ass: main_gen} \& \ref{ass: main_erm}, then the iterations of \eqref{alg: precond} converge to $\bW_{\text{GD}, \infty}$. We refer to such preconditioners as \textbf{span-preserving}. For instance, if $K(\bZ) = \phi(\bZ^T \bZ)$ is convex and also satisfies the assumptions of Theorem \ref{thm: conv_gen}, then $\nabla K(\bZ) = 2\bZ \nabla \phi (\bZ^T \bZ)$ which implies $\nabla K(\nabla \calL (\bW_t)) \in \calS$ for all $t\ge 1$. Hence, such $K(\bZ)$ exhibits the same implicit bias as gradient descent. We will discuss an example of such $K(\bZ)$ in Section IV.
\end{corollary}

Several remarks are in order.

\begin{remark}
    Note that Theorem \ref{thm: conv_gen} only proves the existence of the convergence point. Whether $\bW_{\infty}$ is independent of the step size $\eta$ requires further investigation. As we will see in the experiments section, this is not generally the case. 
\end{remark}

\begin{remark}
    Essentially, the third part of Theorem \ref{thm: conv_gen} states that in order to obtain a different solution from that of GD, the value of the loss at initialization should be non-negligible; in other words and roughly speaking, using the update rule \eqref{alg: precond} for fine-tuning, when the initialization point is already close to the GD solution, does not necessarily yield a qualitatively different solution from that of GD. 
\end{remark}

\begin{remark}
In the bounds provided in the third part of Theorem \ref{thm: conv_gen}, letting $\eta \rightarrow 0$ yields the coefficients $\frac{2\sqrt{2}
}{\tilde{\mu} } + \frac{2(L_K+2)}{\alpha}, 1+\frac{2\sqrt{2}
}{\tilde{\mu} } + \frac{2(L_K + 2)}{\alpha} < \infty$, which are non-vacuous.
\end{remark}

\begin{remark}
    Note that Assumptions \ref{ass: main_gen}.3 and the convexity of $(K - \alpha \calL^\ast)(\bZ)$ on $\calB_r \cap \calS$ are similar to the relative smoothness and strong convexity assumptions of \cite{maddison2021dual}. There are two important distinctions. First, the convergence of the loss does not imply the convergence of the weights as $\operatorname{null}(\bX) \neq \emptyset$. In other words, $\bX \bW_t \rightarrow \bY$ does not imply the convergence of $\bW_t \rightarrow \bW_{\infty}$ in the overparameterized regime. This is the case in \cite{maddison2021dual,laude2025anisotropic, oikonomidis2025nonlinearly,oikonomidis2025nonlinearlymom}. The second difference is that  $(K - \alpha \calL^\ast)(\bZ)$ has to be convex locally, by virtue of the first part of Theorem \ref{thm: conv_gen}, $\calB_r$ always exists. As we will see, for the squared loss, this relaxation allows for preconditioners which are not globally strongly convex.
\end{remark}

\subsection{Isotropic Preconditioner}

For the isotropic preconditioner, one can go beyond the full-batch version and consider the stochastic variant of the algorithm as in \eqref{alg: stoch} and precisely characterize the convergence point. Let $\bb_t$ be a random subset of data indices of cardinality $B$, selected independently of previous steps at iteration $t$. Let $\calF_t$ be the $\sigma$-algebra generated by $\bb_0, \cdots, \bb_t$.

\begin{theorem}\label{thm: conv_stoc}
    (Isotropic Preconditioner) Assume that $K(\cdot)$ is isotropic, that is, $ K(\cdot) = h(\|\cdot\|_F)$ for some differentiable increasing convex $h:\bbR \rightarrow \bbR$ satisfying $h'(0)=0$. Suppose the sampling be such that $\bbE [\calL_{t} ^{(B)} (\bW_{t-1}) | \calF_{t-1} ]= \calL(\bW_{t-1})$. Then for $1 \le B \le n$, under Assumptions \ref{ass: main_gen} and \ref{ass: main_erm}, the following statements hold.
    \begin{enumerate}[I)]
        \item If $\eta \le \frac{1}{L_K \tilde{M}_B}$, the iterations always remain bounded, that is, with probability 1, for any $\bW_{\ast} \in \calM$ \begin{align*}
        \|\bW_{t} -  \bW_{\ast} \|_F^2 \le \|\bW_{0} -  \bW_{\ast}\|^2_F, \quad t \ge 1
    \end{align*} 
        Let $r' := \tilde{M}_B \|\bW_{0} -  \bW_{\text{GD},\infty}\|_F $.
        \item For
        $m_K:= \inf_{0 < \|\bZ\|_F \le r'} \frac{h'(\|\bZ\|_F)}{\|\bZ\|_F} $, if $m_K > 0$ then,
          \begin{align}\label{ineq: lin_conv}
           & \bbE \|\bW_{t} - \bW_{\text{GD},\infty}\|_F^2 \nonumber \\
           &\le  \|\bW_{0} - \bW_{\text{GD},\infty}\|_F^2 \Bigl(1 + \eta^2 L_K^2 \tilde{M}_B^2 - 2\eta  m_K \tilde{\mu}  \Bigr)^t 
        \end{align}
        Here, the expectation is taken with respect to the randomness in the batch selection.
    \end{enumerate}

\end{theorem}
\begin{remark}
     $\eta^\ast = \frac{B^2}{n} \cdot \frac{m_K}{L^2_K}\cdot \frac{\mu}{M^2} \cdot \frac{\sigma_n(\bX \bX^T)}{\sigma^2_1(\bX \bX^T)}$ minimizes the right hand side of \eqref{ineq: lin_conv} and the upper bound reduces to
        \begin{align*}
             \bbE\|\bW_{t} - \bW_{\infty}\|_F^2 &\le  \|\bW_{0} - \bW_{\infty}\|_F^2 \\
             &\cdot \Bigl(1-\frac{m_K^2}{L_K^2} \frac{B^2}{n^2} \frac{\mu^2}{M^2} \frac{\sigma^2_n(\bX \bX^T)}{\sigma^2_1(\bX \bX^T)}\Bigr)^t
        \end{align*}
\end{remark}

We observe that using a larger batch size improves the convergence rate, and among various losses $\calL$, the squared loss $\ell := \|\cdot\|_F^2$ yields the best $\frac{\mu}{M} = 1$ ratio. Note, however, that the optimal learning is smaller for the non-quadratic preconditioners. In Fig. \ref{fig:compare_lin} we compare the convergence rates of several preconditioners and we observe that their rates are similar. 


We defer the precise description and discussion to the Section \ref{ex4}.

\begin{figure}
    \centering
    \includegraphics[width=0.8\linewidth]{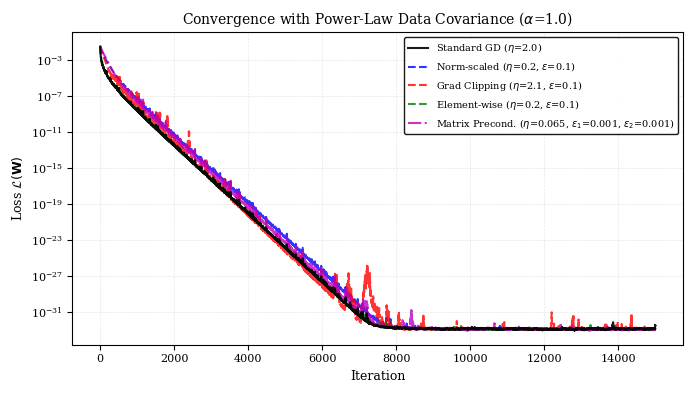}
    \caption{Convergence of the loss across iterations for various preconditioner algorithms. Here $\epsilon$ and $\eta$ are tuned to obtain a relatively good rate compared to that of GD.}
    \label{fig:compare_lin}
\end{figure}

\begin{remark}
     In \cite{ma2018power}, the authors derive the convergence of SGD for convex losses with arbitrary batch-size in the interpolation regime. Our results extend this result to isotropic preconditioners; we leave the investigation of optimality of the rate for future work.
\end{remark}

\subsection{Analytical Example}
In this section, we provide an analytical example where the choice of the constant step-size alters the point of convergence of a nonlinear preconditioner. Now, let us consider the following problem where $\bx, \bw \in \bbR^d, y \in \bbR$.
\begin{align*}
    \min_{\bw \in \bbR^d} \frac{1}{2} (\bx^T \bw - y)^2
\end{align*}
We choose the preconditioner $   K(\bz) := \frac{\epsilon}{2} \|\bz\|_2^2 + \sum_{i=1}^d F(z_i)$ where
\begin{align*}
    F(z; \epsilon):=\frac{1}{2}\begin{cases}
        (|z|-\epsilon)^2 & |z| \ge \epsilon \\
        0 & |z| < \epsilon
\end{cases}.
\end{align*}
Therefore, $\operatorname{ST}(z;\epsilon) = \partial F(z;\epsilon)/\partial z$ where $\operatorname{ST}(\cdot)$ is the classical soft-thresholding function \cite{donoho1995noising}. Then, it is immediate to see that
\begin{align*}
    \bw_t &= \bw_{t-1} \\
    &-  \eta \Bigl(\epsilon (\bx^T \bw_{t-1} - y) \bx + \operatorname{ST}( (\bx^T \bw_{t-1} - y) \bx;\epsilon)\Bigr) 
\end{align*}
If $|\bx^T \bw_{t-1} - y| \le \frac{\epsilon}{\|\bx\|_{\infty}}$, then $\operatorname{ST}(|\bx^T \bw_{t-1} - y| \bx ; \epsilon) = \bzero$, and therefore
\begin{align*}
    \bx^T \bw_{t} - y &= \bx^T \bw_{t-1} - y - \eta \epsilon \|\bx\|_2^2 (\bx^T \bw_{t-1} - y) \\
    &= (1-\eta \epsilon \|\bx\|_2^2) (\bx^T \bw_{t-1} - y)
\end{align*}
Taking $\eta < \frac{1}{\epsilon \|\bx\|_2^2}$ ensures that $|\bx^T \bw_t - y| \le \frac{\epsilon}{\|\bx\|_{\infty}}$. Therefore, for all $t \ge T$ such that $T$ is the first iteration at which $|\bx^T \bw_{T} - y| \le \frac{\epsilon}{\|\bx\|_{\infty}}$, the algorithm essentially becomes gradient descent with a fixed step-size initialized at $\bw_T$ which depends on $\eta$. By the implicit bias property of GD as in Fact \ref{fact: gd_implicit}, this converges to
\begin{align*}
    \bw_{\infty} = \bw_T - \frac{\bx^T \bw_T - y}{\|\bx\|_2^2} \bx
\end{align*}
For a concrete example, take $d=2$, $\bx=(1,2), y = 0$ and let $\epsilon = 1$, $\bw_0 = (1,1)$, then the step-size  must satisfy $\eta < \frac{1}{5}$. Furthermore, in the first step, we have
\begin{align*}
     |\bx^T\bw_{t-1} - y| \bx = \begin{pmatrix}
        3 \\
        6
    \end{pmatrix}, \quad \text{ST}(|\bx^T \bw_{t-1} - y| \bx ; \epsilon) = \begin{pmatrix}
        2 \\
        5
    \end{pmatrix}
\end{align*}
This implies $\bw_1 = \begin{pmatrix}
    1-5 \eta \\
    1- 11 \eta
\end{pmatrix}$ and
\begin{align*}
    \bx^T \bw_1 = \bx^T \bw_0 - \eta ( \bx^T \begin{pmatrix}
        3 \\
        6
    \end{pmatrix} + \bx^T \begin{pmatrix}
        2 \\
        5
    \end{pmatrix}) = 3 - 27 \eta 
\end{align*}
Requiring this to be less than $\frac{1}{2}$ implies implies $\frac{5}{54} < \eta < \frac{7}{54}$. With such $\eta$, the algorithm converges to
\begin{align*}
    \bw_{\infty} = \frac{1}{5}\begin{pmatrix}
            2 + 2\eta \\
            -1 - \eta
    \end{pmatrix}
\end{align*}
Hence we have the following proposition.
\begin{Proposition}
    There exists a convex loss $\calL$ and a strongly convex dual space preconditioner $K$ such that the map $\eta \mapsto \bW_{\infty, \eta}$ where $ \bW_{\infty, \eta}$ denotes the convergence point of algorithm \eqref{alg: precond} is injective for $\eta \in [a, b] \subset (0, \infty)$.
\end{Proposition}

\cite{gunasekar2018characterizing} empirically noted that, for the natural-gradient descent, the implicit bias depends on the choice of the constant step size. In natural-gradient descent ($\bW_t = \bW_{t-1} - \nabla^2 \psi(\bW_{t-1})^{-1} \nabla \calL(\bW_{t-1})$), the left-preconditioner is a nonlinear mapping of the weights themselves. We leave the exploration of the connection between the dual-space preconditioners and natural-gradient descent for future work.


\section{Examples and Discussion}

To delineate the utility of our results, in this section, we will focus on a few selected examples. We consider the squared loss which amounts to taking $\mu = M = 1$ in Theorems \ref{thm: conv_gen} and \ref{thm: conv_stoc}. First, we demonstrate the utility of Theorem \ref{thm: conv_stoc} and then proceed to analyze a nonisotropic preconditioner using Theorem \ref{thm: conv_gen}.


   

\textbf{1) Normalized Stochastic Gradient Descent:} We consider the case where $$K(\bZ) = \|\bZ\|_F - \epsilon\log(\epsilon+ \|\bZ\|_F) + \epsilon \log(\epsilon)$$
to which we can apply Theorem \ref{thm: conv_stoc} for $h(x) = x - \epsilon \log (x + \epsilon)$ yielding $\frac{h'(x)}{x} = \frac{1}{x+\epsilon}$. By the first part of Theorem \ref{thm: conv_stoc}, there exists some $r':= \tilde{M}_B \|\bW_0 - \bW_\ast\|_F> 0$ such that $\|\nabla \calL_t^{(B)}(\bW_{t-1})\|_F < r'$, hence, $\frac{h'(\|\bZ\|_F)}{\|\bZ\|_F} > \frac{1}{r'+\epsilon}$ for $\bZ \in \calB_{r'}$. Taking $m_K := \frac{1}{r'+\epsilon}$, we observe that the iterations of stochastic preconditioned GD with batch size $B$ converge linearly for $\eta < \frac{B^2}{n} \frac{\epsilon^2}{\frac{\sigma_1(\bX \bX^T)}{B}\|\bW_0 -  \bW_{\text{GD},\infty}\|_F + \epsilon}  \frac{\sigma_n(\bX \bX^T)}{\sigma^2_1(\bX \bX^T)}$ the iterations 
\begin{align*}
    \bW_{t} = \bW_{t-1} - \eta \frac{\nabla \calL^{(B)}_t(\bW_{t-1})}{\epsilon + \|\nabla \calL^{(B)}_t(\bW_{t-1})\|_F}
\end{align*}

converge to the optimal solution of \eqref{opt: GD}. Here we have taken $L_K = \frac{1}{\epsilon}$.

\textbf{2) Stochastic Gradient Clipping:} Gradient Clipping is another popular scheme \cite{zhang2019gradient} which can be phrased as 
\begin{align*}
    K(\bZ) = \begin{cases}
        \frac{\|\bZ\|_F^2}{2} & \|\bZ\|_F\le \epsilon \\
        \epsilon \|\bZ\|_F - \frac{\epsilon^2}{2} & \|\bZ\|_F > \epsilon
    \end{cases}
\end{align*}
Which yields $\nabla K( \bZ) = \min \{\frac{\epsilon}{\|\bZ\|_F}, 1\} \bZ$. 
Here $\frac{h'(\|\bZ\|_F)}{\|\bZ\|_F} = \min \{1, \frac{\epsilon}{\|\bZ\|_F}\}$. With the same radius $r'$ as before, there is a non-zero $m_k:= \min \{\frac{B\epsilon}{\|\bW_0 -  \bW_{\text{GD},\infty}\|_F\sigma_1(\bX \bX^T)}, 1\} \le \frac{h'(\|\bZ\|_F)}{\|\bZ\|_F}$ for $\bZ \in \calB_{r'}$. Note that $K(\cdot)$ is differentiable and its gradient is $1$-Lipschitz, so $L_K =1$, hence for $$\eta \le \frac{B^2}{n} \frac{\sigma_n(\bX \bX^T)}{\sigma^2_1(\bX \bX^T)} \min\Bigl\{ \frac{B\epsilon}{\|\bW_0 -  \bW_{\text{GD},\infty}\|_F \sigma_1(\bX \bX^T)},1\Bigr\},$$
the iterations
\begin{align*}
    \bW_{t} = \bW_{t-1} - \eta \min\Bigl\{\frac{\epsilon}{\|\nabla \calL^{(B)}_t(\bW_{t-1})\|_F}, 1\Bigr\} \nabla \calL^{(B)}_t(\bW_{t-1})
\end{align*}
converge to the optimal solution of \eqref{opt: GD}.

\textbf{3) Adam:} To illustrate the utility of the second part of Theorem \ref{thm: conv_gen}, let us consider $K(\bZ) = \sum_{j,j'=1}^{d,k} |\bZ^{(j, j')}| - \epsilon \log(1+ \frac{|\bZ^{(j, j')}}{\epsilon}|)$ which yields $\nabla K(\bZ) = \frac{1}{\epsilon + |\bZ|} \bZ$, where $|\cdot|$ is the absolute value function acting entrywise and $|\bZ|>0$. The update rule in \eqref{alg: precond} turns into
\begin{align}
    \bW_{t} = \bW_{t-1} - \eta \frac{\nabla \calL(\bW_{t-1})}{\epsilon + |\nabla \calL(\bW_{t-1})|}\label{alg: precond_adam}
\end{align}
Note that the choice of $\nabla K(\cdot)$ follows the update rule of Adam \cite{kingma2014adam} without weight decay and momentum.
Generally, the weights are initialized such that for every $(j,j') \in [d] \times [k]$, we have that $|\nabla \calL(\bW_0)^{(j,j')}| > 0$. Therefore, we may assume that throughout the iterations, no entry of the gradient of the loss is exactly zero. Setting $m_K = \frac{\epsilon}{(\epsilon + \sup_{t\ge 1} \|\nabla \calL(\bW_t)\|_{\infty}}$ and $\alpha < \min\{\frac{1}{\eta}, \frac{1}{n m_K} \}$,  $\eta < n \epsilon$, we can utilize Theorem \ref{thm: conv_gen} to conclude convergence and to bound the distance between $\bW_{\infty}$ and $\bW_{\text{GD}, \infty}$. 

We would also like to provide intuition for the dynamics of \eqref{alg: precond_adam}. We observe that, in the beginning, if the entries of $|\nabla \calL(\bW_{t-1})| \gg \epsilon$, then $\frac{\nabla \calL(\bW_{t-1})}{\epsilon + |\nabla \calL(\bW_{t-1})|} \approx \sign(\nabla \calL(\bW_{t-1}))$. Toward the end, where $|\nabla \calL(\bW_{t-1})| \ll \epsilon$, we have $\frac{\nabla \calL(\bW_{t-1})}{\epsilon + |\nabla \calL(\bW_{t-1})|} \approx \frac{1}{\epsilon} \nabla \calL(\bW_{t-1})$. To summarize, the updates are very close to SignGD \cite{bernstein2018signsgd} in the beginning phase of the optimization and are very close to GD in the end phase.  


It is worth mentioning the related work \cite{xie2024implicit}, which considers AdamW. The main difference between their setting and ours is that the setting of \cite{xie2024implicit} includes weight decay. The presence of weight decay often drives the solutions away from interpolation and leads to a substantially different analysis. 
    

\textbf{4) Smoothed Matrix Elastic-Net} \label{ex4}: Let 
\begin{align*}
    K(\bZ) = \tr((\bZ^T \bZ + \epsilon_1 \bI)^{1/2}) + \frac{\epsilon_2}{2} \|\bZ\|_F^2 - k \sqrt{\epsilon_1}
\end{align*}
Intuitively, the first term serves as an approximation of the nuclear norm of the matrix $\bZ$. Its gradient is
\begin{align*}
    \nabla K(\bZ) = \bZ (\bZ^T \bZ + \epsilon_1 \bI)^{-1/2} + \epsilon_2 \bZ
\end{align*}
By the Daletskii-Krein theorem (Chapter V in \cite{bhatia2013matrix}) applied to the Hermitian dilation of $\nabla K(\bZ)$, we observe that the first term is $\frac{1}{\sqrt{\epsilon_1}}$-Lipschitz and therefore $L_K = \frac{1}{\sqrt{\epsilon_1}}+\epsilon_2$. Clearly $K(\bZ)$ is $\epsilon_2$-strongly convex, hence the Assumption $K-\alpha \calL^\ast$ can be satisfied globally. Therefore, Theorem \ref{thm: conv_gen} implies that the iterates of
\begin{align*}
    &\bW_t = \bW_{t-1} \\
    &- \eta \nabla \calL(\bW_{t-1}) \Bigl((\nabla \calL(\bW_{t-1}) ^T \nabla \calL(\bW_{t-1})  + \epsilon_1 \bI)^{-1/2} + \epsilon_2 \bI\Bigr)
\end{align*}
converge to $\bW_{\text{GD},\infty}$ as defined in \eqref{opt: GD}.

\subsection{Discussion}

In Fig. \ref{fig:compare_lin}, we demonstrate the linear convergence of the squared loss against the iterations. We consider a linear regression task with a Gaussian matrix with a power-law covariance $\bSigma$ where $\bSigma_{ii} = i^{-\alpha}$. We observe that among the algorithms of the previous section, all have similar convergence rates to that of standard SGD, despite using either clipped or normalized updates. Whether adding momentum to the argument of $\nabla K(\cdot)$ or to the weights can improve the theoretical rate is an interesting direction for future work.

Motivated by the setting of \cite{oymak2019overparameterized}, we consider the training of a two-layer MLP with squared loss on a linear regression task. We observe that our proposed approximation of matrix elastic-net yields a faster convergence rate than SGD and entrywise normalization.

\begin{figure}
    \centering
    \includegraphics[width=0.9\linewidth]{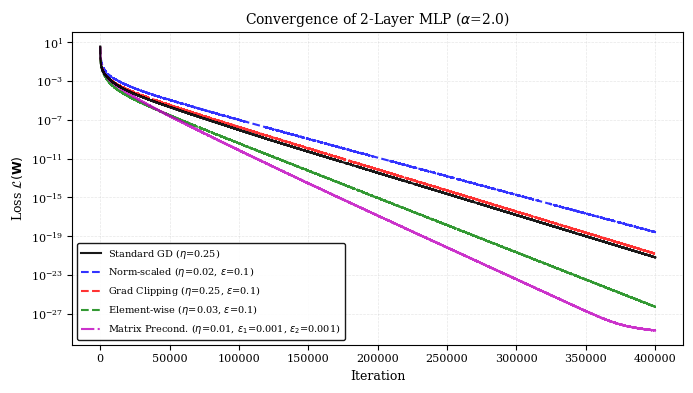}
    \caption{Performance of various preconditioners on a single-hidden-layer MLP. The parameters here are tuned to demonstrate the speedup over the best performance of vanilla gradient descent.}
    \label{fig:compare_mlp}
\end{figure}

\section{Outline of the proofs}







\subsection{Outline of the Proof of Theorem 1.II}

To prove Theorem 1, first, we use Proposition \ref{prop: fund_id} to prove that $\|\nabla \calL(\bW_t)\|_2$ remains bounded. Then, we proceed to show the sequence $\{ \bW_t \}_{t=1}$ is Cauchy with the aid of Lemma \ref{lem: nablaL_bnd} which immediately yields a convergence rate. Finally, using the Assumptions \ref{ass: main_gen} and \ref{ass: main_erm}, we prove that $\bW_{\infty} \in \calM$.

\subsection{Outline of the Proof of Theorem 1.III}

In the case that $K(\cdot)$ is not isotropic, for instance, in the third example of Section IV, we cannot apply the reasoning as in Theorem \ref{thm: conv_stoc}. To this end, we introduce the following auxiliary variable $\hbW$. Consider $\{\bW_t\}_{t=1}$ generated according to the update rule in \eqref{alg: precond}, initialized from $\bW_0$.
Furthermore, consider
\begin{align*}
    \hbW_t =  \bW_{t-1} - \eta \nabla\calL(\bW_{t-1})
\end{align*}
Consider the following GD updates, initialized from the same $\bW_0$:
\begin{align*}
    \bW_{\text{GD},t} =  \bW_{\text{GD},t-1} - \eta \nabla\calL(\bW_{\text{GD},t-1})
\end{align*}

We use $\hbW_t$ to analyze $\|\bW_{\text{GD},t}- \bW_t\|_F$. By algebraic manipulations, we observe that
\begin{align*}
    \|&\bW_{\text{GD},i}  - \hbW_{t}\|_{F}  \le \sum_{j=1}^{t-1}  \|\bW_{j} - \hbW_{j}\|_{F} \\ 
    &+  \eta \sum_{j=1}^{t} \Bigl(\|\nabla\calL( \bW_{\text{GD},j-1}) \|_{F} + \| \nabla\calL(\bW_{j-1})  \|_{F} \Bigr)
\end{align*}
Combining the triangle inequality and $L_K$-Lipschitz continuity of $\nabla K$, we obtain
\begin{align*}
    \|\hbW_t - \bW_t\|_{F} &= \eta \Bigl\| \nabla K\left(\nabla\calL(\bW_{t-1})\right) - \nabla\calL(\bW_{t-1}) \Bigr\|_{F}\\
    &\le \eta (L_K +1) \| \nabla\calL(\bW_{t-1})\|_{F}
\end{align*}
Combining this identity with the bounds in Lemma \ref{lem: nablaL_bnd}, yields the results of Theorem \ref{thm: conv_gen}.

\subsection{Implicit bias of span-preserving preconditioner}

In this section, we provide an explanation of why the convergence point of the gradient descent as in Fact \ref{fact: gd_implicit} is a natural candidate for the final point of span-preserving preconditioners such that $\nabla K(\bZ) = \bZ \tilde{K}(\bZ)$. Note that for loss functions satisfying Assumptions \ref{ass: main_erm}, we may write $\nabla \calL_t^{(B)} (\bW_{t-1}) = \bX^{(B)T}_t \bL_{t-1}$ where
\begin{align*}
    \bL_{t-1} = \Bigl[\frac{1}{n}\sum_{j=1}^{n} \ell'_{t_j,j'} \left((\bx_{t_j}^{(t-1)})^T \bW_{t-1}^{(j')}-y_{t_j j'}\right) \Bigr]_{j'=1}^k
\end{align*}
This implies that
\begin{align*}
    \nabla K(\nabla \calL_t^{(B)} (\bW_{t-1})) = \bX^T \bLambda_{t-1}
\end{align*}
where $\bLambda_{t-1} := \bL_{t-1} \tilde{K}(\nabla \calL_t^{(B)} (\bW_{t-1}))$. Consequently, $\bW_t - \bW_0 \in \calS$ and it follows that, if the algorithm converges to some $\bW_{\infty} \in \calM$, one must have simultaneously $\bW_{\infty} - \bW_0 \in \calS$ while $\bX \bW_{\infty} = \bY$. $\bW_{\text{GD}, \infty}$ is the only point satisfying the preceding identities.

\section{Experiments}

\subsection{A Simple Case Study of Example 3}
Similar to the analytical worked-out example, we consider the loss
\begin{align*}
    \min_{\bw \in \bbR^2} (\bx^T \bw - \by)^2
\end{align*}
Take $\bx=(1,2)$ and let $\epsilon = 1$, $\bw_0 = (1,1)$, but this time, we use the following preconditioner
\begin{align*}
    \bW_{t} = \bW_{t-1} - \eta \frac{\nabla \calL(\bW_{t-1})}{\epsilon + |\nabla \calL(\bW_{t-1})|}
\end{align*}
We have plotted the trajectories of the algorithm for various constant step sizes in Fig. \ref{fig:placeholder} near their limiting points. We observe that despite the loss attaining $0$, the convergence point varies with the choice of the step size $\eta$. This serves as another example of the implicit bias of the dual-space preconditioners depending on the value of the constant step-size.

\begin{figure}[t]
    \centering
    \includegraphics[width=1\linewidth]{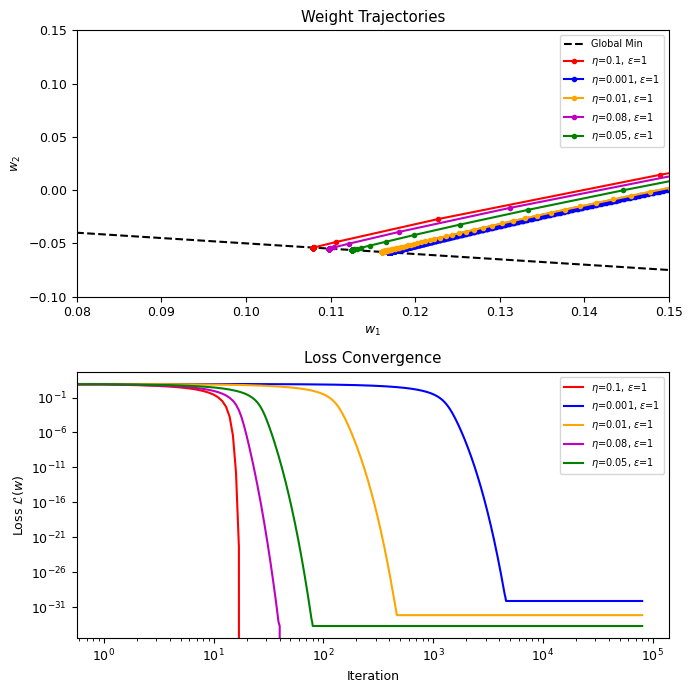}
    \caption{Trajectories and loss curves for the iterates of \eqref{alg: precond_adam} for several choices of constant step sizes with $\epsilon=1$. We observe that varying the step size yields different final interpolating points on the interpolating manifold $w_1 + 2 w_2 = 0$.}
    \label{fig:placeholder}
\end{figure}
\subsection{Higher dimensional example}

\begin{figure}[t]
  \centering
  \includegraphics[width=1.0\linewidth]{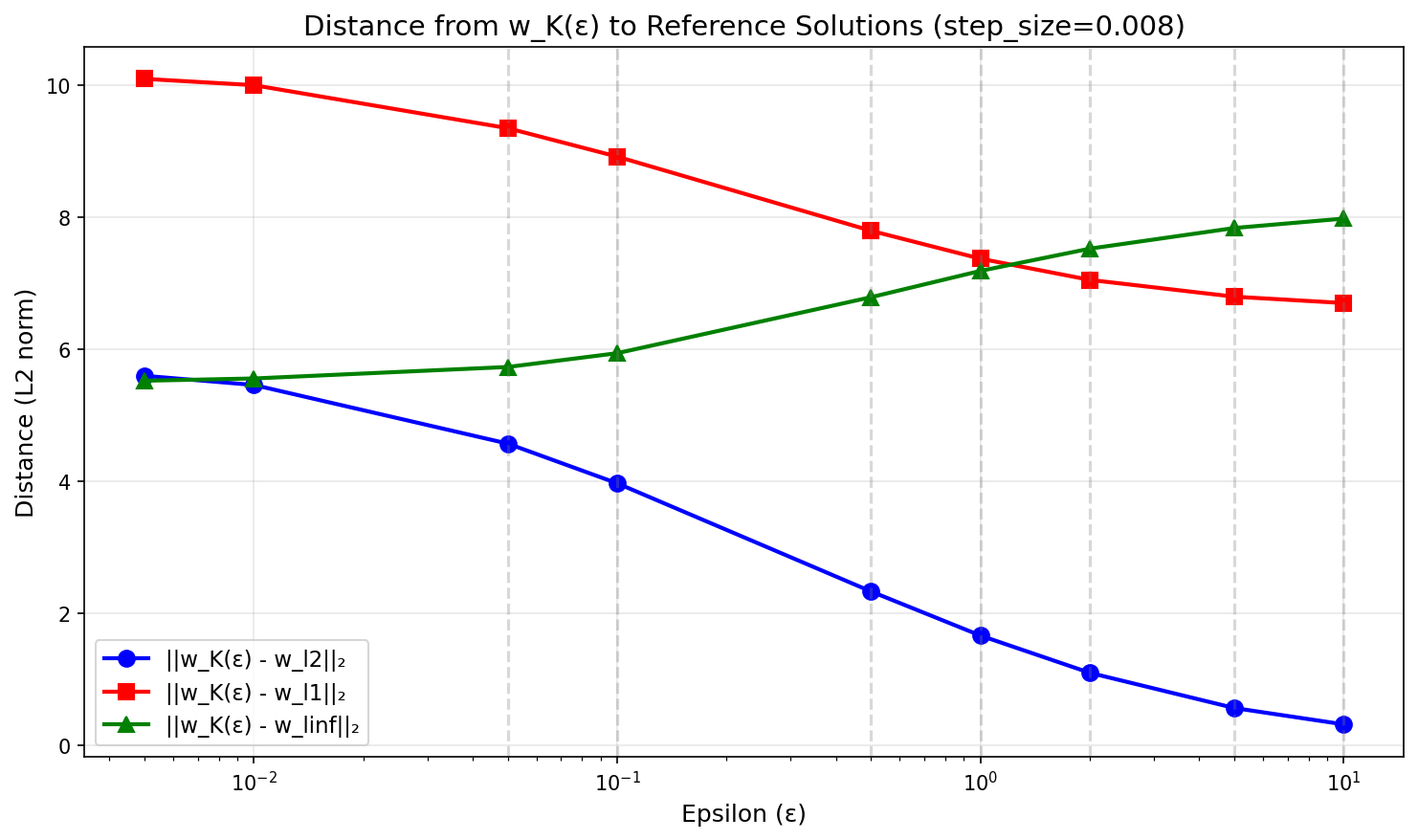}
  \caption{Distance from $p = 1, 2, \infty$ solutions in \eqref{exp: lp} for the solution obtained using \eqref{exp: precond}.}
\label{fig:LR}
\end{figure}
In this section, we mainly consider the following preconditioner without momentum: \begin{equation}\label{exp: precond}
    K(\bZ) = \sum_{j,j'=1}^{d,k} |\bZ^{(j, j')}| - \epsilon\log(\epsilon + |\bZ^{(j, j')}|)
\end{equation}
and for simplicity, we take $k=1$ and consider the squared loss. We consider $\bX \in \bbR^{100 \times 200}$ with i.i.d. standard normal entries. Let $\bw_{\ast} \sim \calN(\bzero, \bI)$ and $\by = \bX \bw_{\ast} + \bz$ where $\bz \sim \calN(\bzero, 0.01\bI)$. For the initialization, we choose $\bw_0 \sim \calN(\bzero, \frac{1}{200} \bI)$ so that $\|\bw_0 \|_2 \approx 1$, with high probability. We fix $\epsilon = 0.5$, vary $\eta$ and compute the corresponding $\|\bW_{\infty} - \bW_{\text{ref}}\|_2$, where $\bW_{\text{ref}}$ is the convergence point for $\eta = 0.005$. In Fig. \ref{fig:epsilon}, we observe that despite attaining the same distance from $\bW_{\text{GD},\infty}$, $\bW_{\infty}$ depends on $\eta$. This is in contrast with the case of SMD, where the existing works \cite{soudry2018implicit, azizan2018stochastic, akhtiamov2026implicitbiasconvergencematrix} show that the implicit bias is independent of the choice of the step size, provided that it is sufficiently small. 

\subsection{Comparison with $\ell_p$-Norm Solutions}

In Fig. \ref{fig:LR}, for a fixed step size $\eta$, we vary $\epsilon$ and plot the distance of convergence point of \eqref{alg: precond} from the optimal point of 
\begin{align}\label{exp: lp}
    \min_{\bW \in \bbR^d} \|\bW - \bW_{0}\|_p \quad \text{ s.t. }
    \bX \bW = \bY
\end{align}
for the $p =1, 2, \infty$-norms in the objective as their solutions correspond to three different points on the interpolating manifold. As a result of increasing $\epsilon$, as discussed in Section IV, the iterations of \eqref{exp: precond} will be qualitatively close to those of GD. Interestingly, for small $\epsilon$, we observe that the point of convergence has the same distance from the $p=2$ and $p=\infty$ solutions. Nonetheless, as the solution of \eqref{exp: lp} for $p=2$ is closer to $\bW_{\infty}$ than the other two, Fig. \ref{fig:LR} motivates the study of the $\ell_2$ distance of $\bW_{\infty}$ from $\bW_{\text{GD},\infty}$.

\begin{figure}[t]
  \centering
\includegraphics[width=1.0\linewidth]{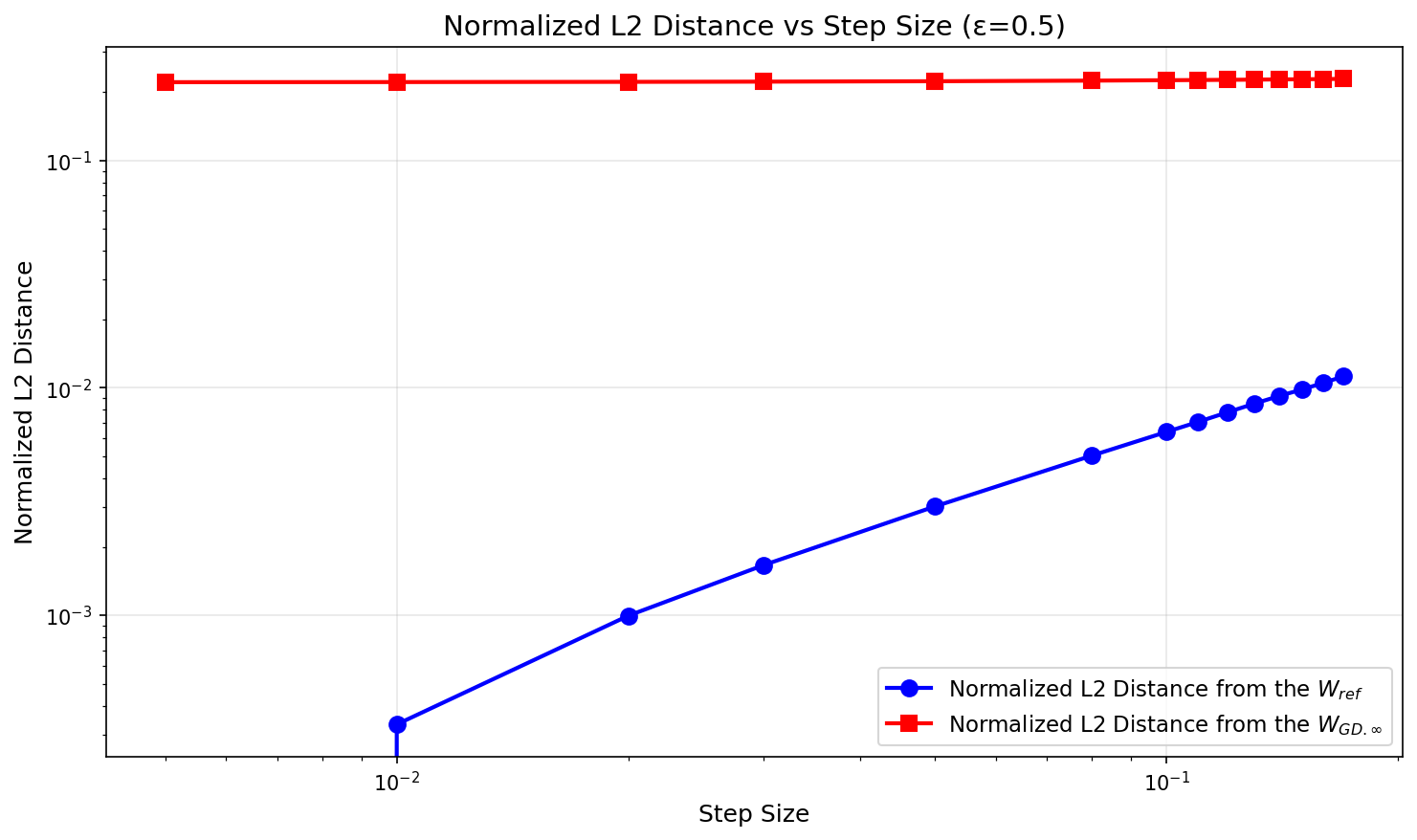}
  \caption{The distance of solutions obtained using \eqref{exp: precond} from $\bW_{\text{GD},\infty}$ and $\bW_{\text{ref}}$ normalized by the norm of $\bW_{\text{ref}}$.} \label{fig:epsilon}
\end{figure}


\section{Conclusion}
In this work, we considered nonlinear preconditioners in the overparameterized regime. Future directions include extending the current analysis to nonsmooth preconditioners which would encompass algorithms such as Muon \cite{jordan2024muon}, providing the exact characterization of the point of convergence for nonisotropic preconditioners and adapting our proof of convergence to the {\it stochastic} nonisotropic preconditioned gradient descent. Studying the convergence and the implicit bias of preconditioners on nonlinear models similar to the setting of \cite{oymak2019overparameterized} is another prominent direction.

\appendix

\section{Details of the Experiments}

\subsection{Proof of Theorem 1.I}
\begin{proof}
    
First, we show that there exists some $r>0$ such that $\sup_{t\ge 1} \|\nabla \calL(\bW_t)\|_F \le r$. As noted earlier, the proof of convergence is based on the result of Proposition \ref{prop: fund_id}, whose proof is presented in the Auxiliary Results section. Using Proposition \ref{prop: fund_id}, we obtain:

\begin{align}\label{eq: sum_fund}
    \tilde{D}_{\calL}&(\bW,\bW_{0}) = \tilde{D}_{\calL}(\bW,\bW_t) + \eta \sum_{t=1}^T K(\nabla \calL(\bW_{t})) \nonumber \\
     &+\sum_{t=1}^T \Bigl[  \tilde{D}_{\calL}(\bW_t,\bW_{t-1}) - \eta D_{K}(\nabla \calL(\bW_t),\nabla \calL(\bW_{t-1})) \Bigr] \nonumber \\
     &+ \eta \sum_{t=1}^T  D_{K}(\nabla \calL(\bW),\nabla \calL(\bW_{t-1}))
\end{align}
Here $\bW$ is chosen so that $\nabla \calL(\bW) = 0$. By taking $T \rightarrow \infty$, we observe that the left hand side of \eqref{eq: sum_fund} is independent of $t$ and finite, and every term on the right hand side is positive. This follows because of the assumption $\calL^\ast - \eta K$ being convex which implies $ \tilde{D}_{\calL}(\bW_t,\bW_{t-1}) - \eta D_{K}(\nabla \calL(\bW_t),\nabla \calL(\bW_{t-1}))$ is nonnegative by Lemma \ref{lem: breg_k_ell}. Therefore, we have that $\sum_{t=1}^{\infty} K(\nabla \calL(\bW_{t}))< \infty$ which implies $K(\nabla \calL(\bW_{t})) \rightarrow 0$. As $K(\cdot)$ is coercive, this implies that $\sup_{t\ge 1} \|\nabla \calL(\bW_{t-1}) \|_F< \infty$ and subsequently the sequence is bounded. Thus, for every $t\ge 1$,  $\nabla \calL(\bW_t) \in \calB_r$ by setting $r= \sup_{t\ge 1} \|\nabla \calL(\bW_{t-1}) \|_F$. 
\end{proof}

\subsection{Proof of Theorem 1.II}
\begin{proof}
    
By our assumption, we can find some $\alpha > 0$ such that the assumption of Lemma \ref{lem: nablaL_bnd} is satisfied. 

We prove that the sequence $\{ \bW_t\}_{t=1}$ is Cauchy and the convergence follows by the completeness of $\bbR^{d \times k}$ with the $\|\cdot\|_F$ metric. For $s, t \in \bbN$ and $t \ge s \ge 1$, we have that
\begin{align*}
    \|\bW_s - \bW_t \|_F &= \eta \| \sum_{i=s}^{t-1} \nabla K ( \nabla \calL(\bW_i))\|_F \\
    &\le \eta  \sum_{i=s}^{t-1}  \|\nabla K ( \nabla \calL(\bW_i))\|_F
\end{align*}
Then by Assumption \ref{ass: main_erm} on the $L_K$-Lipschitzness of $\nabla K(\cdot)$ and the fact that $\nabla K(0) = 0$, we have that
\begin{align*}
      \|\bW_s - \bW_t \|_F \le \eta  L_K \sum_{i=s}^{t-1}  \| \nabla \calL(\bW_i)\|_F
\end{align*}

Then by Lemma \ref{lem: nablaL_bnd}, we obtain
\begin{align*}
    \|\bW_s - \bW_t \|_F \le\eta  L_K \sqrt{2\tilde{M} \calL(\bW_0)}   \sum_{i=s}^{t-1} (\sqrt{1 - \alpha \eta})^{i} \\
    = \eta  \frac{L_K(1 - (\sqrt{1 - \alpha \eta})^{t-s})}{1- \sqrt{1 - \alpha \eta}} \sqrt{2\tilde{M} \calL(\bW_0)} (\sqrt{1 - \alpha \eta})^{s}
\end{align*}
This proves that the sequence  $\{ \bW_t\}_{t=1}$ is Cauchy; Therefore, the iterates of the algorithm \eqref{alg: precond} converge to some $\bW_{\infty} \in \bbR^{d \times k}$. Thus, we obtain the following convergence rate by letting $t \rightarrow \infty$:
\begin{align*}
     \|\bW_{\infty} - \bW_s \|_F \le \eta L_K \frac{ \sqrt{2\tilde{M} \calL(\bW_0)}}{1- \sqrt{1 - \alpha \eta}} (\sqrt{1 - \alpha \eta})^{s}
\end{align*}

Next, we show that $\bW_{\infty} \in \calM$. From the first part, since $K(\nabla \calL(\bW_t)) \rightarrow 0$, which implies $\nabla \calL(\bW_t) \rightarrow 0$ by Assumption \ref{ass: main_gen}. Subsequently, $\nabla \calL(\bW_{\infty}) = 0$. From Assumption \ref{ass: main_erm}, we observe $\bX \bW_{\infty} = \bY$; therefore, $\bW_{\infty} \in \calM$.

If $\nabla K(\nabla \calL(\bW_t)) \in \calS$, take $\bL_t \in \bbR^{n \times k}$ such that $\nabla K(\nabla \calL(\bW_t)) = \bX^T \bL_{t}$. Now we may write:
\begin{align*}
    \bW_t - \bW_0 = \eta \bX^{T} \sum_{i=0}^{t-1} \bL_{i}
\end{align*}
Hence 
\begin{align*}
     \bW_{\infty} - \bW_0 = \eta \bX^{T} \sum_{i=0}^{\infty} \bL_{i}
\end{align*}
Now we observe that since $\bX \bW_{\infty} = \bY$ and $\bW_{\infty} - \bW_0  = \bX^T \bLambda$ for $\bLambda :=  \eta \sum_{i=0}^{t-1} \bL_{i}$, $\bW_{\infty}$ satisfies the KKT conditions of \eqref{opt: GD}. As \eqref{opt: GD} has a unique minimizer, $\bW_{\infty}$ is the minimizer of \eqref{opt: GD} and $\bW_{\infty} = \bW_{\text{GD}, \infty}$.
\end{proof}
\subsection{Proof of Theorem 1.III}
\begin{proof}
Consider $\{\bW_t\}_{t=1}$ generated according to the following update rule, initialized from $\bW_0$
\begin{align*}
    \bW_t =  \bW_{t-1} - \eta \nabla K\left(\nabla\calL(\bW_{t-1})\right)
\end{align*}
Furthermore, consider
\begin{align*}
    \hbW_t =  \bW_{t-1} - \eta \nabla\calL(\bW_{t-1})
\end{align*}
And consider the following GD updates, initialized from the same $\bW_0$:
\begin{align*}
    \bW_{\text{GD},t} =  \bW_{\text{GD},t-1} - \eta \nabla\calL(\bW_{\text{GD},t-1})
\end{align*}

Inductively, we have:
\begin{align*}
    \bW_{\text{GD},1} - &\hbW_{1} = 0 \\
    \bW_{\text{GD},2} - &\hbW_{2} \\
    &=  \hbW_{1} - \bW_1 - \eta \Bigl(\nabla\calL( \hbW_{1}) - \nabla\calL(\bW_1)  \Bigr) \\ 
\end{align*}
Then
\begin{align*}
    \|\bW_{\text{GD},2} - \hbW_{2}\|_{F} \le&  \| \hbW_{1} - \bW_1 \|_{F} \\
    &+ \eta \Bigl\|\nabla\calL( \bW_{\text{GD},1}) - \nabla\calL(\bW_1)  \Bigr\|_{F} \\ 
    \le & \| \hbW_{1} - \bW_1 \|_{F} \\
    &+ \eta \|\nabla\calL( \bW_{\text{GD},1})\|_{F} + \eta \| \nabla\calL(\bW_1) \|_{F}
\end{align*}
Similarly, we observe that since
\begin{align*}
    \bW_{\text{GD},3} - &\hbW_{3} \\
    &=  \bW_{\text{GD},2} - \bW_{2} - \eta \Bigl(\nabla\calL( \bW_{\text{GD},2}) - \nabla\calL(\bW_2)  \Bigr)
\end{align*}
Then
\begin{align*}
    \|\bW_{\text{GD},3} - \hbW_{3}\|_{F} \le& \|\bW_{\text{GD},2} - \bW_{2}\|_{F} \\
    &+ \eta \Bigl\|\nabla\calL( \bW_{\text{GD},2}) - \nabla\calL(\bW_2)  \Bigr\|_{F} \\ 
    \le& \|\bW_{\text{GD},2} - \hbW_{2}\|_{F} + \|\bW_2 - \hbW_{2}\|_{F} \\
    &+ \eta \|\nabla\calL( \bW_{\text{GD},2}) \|_F + \eta \| \nabla\calL(\bW_2)  \|_{F}
\end{align*}
Therefore, for any $t$, we have
\begin{align*}
    \|&\bW_{\text{GD},t}  - \hbW_{t}\|_{F} \\
    &\le \|\bW_{\text{GD},t-1} - \hbW_{t-1}\|_{F} + \|\bW_{t-1} - \hbW_{t-1}\|_{F} \\
    &+ \eta \|\nabla\calL( \bW_{\text{GD},t-1}) \|_{F}
    + \eta \| \nabla\calL(\bW_{t-1})  \|_{F} \\ \cdots &  \le \sum_{j=1}^{t-1}  \|\bW_{j} - \hbW_{j}\|_{F} \\ 
    &+ \eta \sum_{j=1}^{t} \Bigl(\|\nabla\calL( \bW_{\text{GD},j-1}) \|_{F} + \| \nabla\calL(\bW_{j-1})  \|_{F} \Bigr)
\end{align*}
Then we note that by $L_K$-Lipschitz continuity of $\nabla K$
\begin{align*}
    \|\hbW_t - \bW_t\|_{F} &= \eta \Bigl\| \nabla K\left(\nabla\calL(\bW_{t-1})\right) - \nabla\calL(\bW_{t-1}) \Bigr\|_{F}\\
    &\le \eta (L_K +1) \| \nabla\calL(\bW_{t-1})\|_{F}
\end{align*}

This implies
\begin{align*}
    \|&\bW_{\text{GD},t}  - \hbW_{t}\|_{F} \\
    & \le \eta \sum_{j=1}^{t} \Bigl(\|\nabla\calL( \bW_{\text{GD},j-1}) \|_{F} + (L_K + 2) \| \nabla\calL(\bW_{j-1})  \|_{F} \Bigr)
\end{align*}

Using Lemma \ref{lem: nablaL_bnd} and taking $t \rightarrow \infty$ yields
\begin{align*}
   & \|\bW_{\text{GD},\infty} - \bW_{\infty}\|_{F} \le  \tilde{M} \eta  \|\bW_0 -\bW_{\text{GD},\infty} \|_{F} \\
    &\cdot \Bigl( \frac{\sqrt{2} }{1-\sqrt{1-\eta \tilde{\mu}}} + \frac{L_K + 2}{1-\sqrt{1-\alpha \eta}}\Bigr)
\end{align*}
This also provides a bound on the convergence point:
\begin{align*}
     &\|\bW_0 - \bW_{\infty}\|_{F} \le    \|\bW_0 -\bW_{\text{GD},\infty} \|_{F} \\
    &\cdot \Bigl(1+ \frac{\sqrt{2} \tilde{M} \eta }{1-\sqrt{1-\eta\tilde{\mu}}} + \frac{\tilde{M} \eta (L_K + 2) }{1-\sqrt{1-\alpha \eta}}\Bigr)
\end{align*}
\end{proof}

\subsection{Proof of Theorem 2.I}

\begin{proof}
    We show that for a small enough step-size $\eta > 0$, 
    \begin{align*}
        \|\bW_{t} - \bW_{\ast} \|_F^2 \le \|\bW_{t-1} - \bW_{\ast} \|_F^2 
    \end{align*}
    Using \eqref{alg: stoch}:
    \begin{align*}
        \|\bW_{t} - \bW_{\ast} \|_F^2 &= \|\bW_{t-1} - \bW_{\ast} \|_F^2  \\
        &+ \eta^2 \|\nabla K (\nabla \calL_t^{(B)}(\bW_{t-1}))\|_F^2 \\ 
        &- 2 \eta \tr \Bigl((\bW_{t-1} - \bW_{\ast})^T \nabla K ( \nabla\calL^{(B)}_{i}(\bW_{t-1})) \Bigr)
    \end{align*}
    Now we observe that for the isotropic preconditioner as $\nabla \calL^{(B)}_t(\bW_{\ast}) = \bzero$:
    \begin{align*}
         \tr \Bigl((\bW_{t-1} - \bW_{\ast})^T \nabla K ( \nabla\calL_{t}^{(B)}(\bW_{t-1})) \Bigr) \\
         = \frac{h'(\|\nabla \calL_t^{(B)}(\bW_{t-1})\|_F)}{\| \nabla \calL_t^{(B)}(\bW_{t-1})\|_F} \tr \Bigl((\bW_{t-1} - \bW_{\ast})^T  \nabla\calL_{t}^{(B)}(\bW_{t-1}) \Bigr) \\
         \ge \frac{h'(\|\nabla \calL_t^{(B)}(\bW_{t-1})\|_F)}{\|\nabla \calL_t^{(B)}(\bW_{t-1})\|_F} \frac{\|\nabla \calL_t^{(B)}(\bW_{t-1})\|_F^2}{2 \tilde{M}_B}
    \end{align*} 
    As the sampling is done without replacement, the last inequality follows from the $\tilde{M}_B := \frac{M \sigma_1(\bX \bX^T)}{B}$-smoothness of the loss and Lemma \ref{lem: gen_dec_lem}. Thus 
    \begin{align*}
        \|\bW_{t} - \bW_{\ast} \|_F^2 \le \|\bW_{t-1} - \bW_{\ast} \|_F^2 \\
        + \eta  h'(\|\nabla \calL_t^{(B)}(\bW_{t-1})\|_F) \|\nabla \calL_t^{(B)}(\bW_{t-1})\|_F \\
        \cdot \Bigl( \eta  \frac{h'(\|\nabla \calL_t^{(B)}(\bW_{t-1})\|_F)}{\|\nabla \calL_t^{(B)}(\bW_{t-1})\|_F} - \frac{1}{\tilde{M}_B}  \Bigr) 
    \end{align*}
    Since $\frac{h'(\|\nabla \calL_t^{(B)}(\bW_{t-1})\|_F)}{\|\nabla \calL_t^{(B)}(\bW_{t-1})\|_F} \le L_K$, letting $\eta \le \frac{1}{L_K \tilde{M}_B}$ yields
    \begin{align*}
        \|\bW_{t} - \bW_{\ast} \|_F^2 \le \|\bW_{0} - \bW_{\ast}\|^2_F
    \end{align*}
\end{proof}

\subsection{Proof of Theorem 2.II}
\begin{proof}
    
    Consider $K(\cdot) = h(\|\cdot\|_F)$ and the point $\bW_{\text{GD},\infty}$ defined in \eqref{opt: GD}. We have
    \begin{align*}
        \bW_{t} - \bW_{\text{GD},\infty} &= \bW_{t-1} - \bW_{\text{GD},\infty} \\
        &- \eta \frac{h'(\|\nabla \calL_t^{(B)}(\bW_{t-1})\|_F)}{\|\nabla \calL_t^{(B)}(\bW_{t-1})\|_F} \nabla \calL_t^{(B)}(\bW_{t-1})
    \end{align*}
    Hence
    \begin{align*}
        &\|\bW_{t} - \bW_{\text{GD},\infty}\|^2_F = \|\bW_{t-1} - \bW_{\text{GD},\infty}\|_F^2 \\
        &+ \eta^2 (h'(\|\nabla \calL_t^{(B)}(\bW_{t-1})\|^2_F))^2 \\
        &- 2 \eta \frac{h'(\|\nabla \calL_t^{(B)}(\bW_{t-1})\|_F)}{\|\nabla \calL_t^{(B)}(\bW_{t-1})\|_F} \\
        &\cdot \tr \Bigl( (\bW_{t-1} - \bW_{\text{GD},\infty})^T \nabla \calL_t^{(B)}(\bW_{t-1})\Bigr) 
    \end{align*}
    For the second term, we use the fact that $h(\cdot)$ has $L_K$-Lipschitz derivative and $\nabla \calL_t^{(B)}(\bW_{t-1})$ also has $\tilde{M}_B$-Lipschitz gradient. Hence
    \begin{align*}
         (h'(\|\nabla \calL_t^{(B)}&(\bW_{t-1})\|^2_F))^2 \\
         &\le L_k^2 \|\nabla \calL_t^{(B)}(\bW_{t-1}) - \nabla \calL_t^{(B)}(\bW_{\text{GD},\infty})\|_F^2 \\ 
         &\le L^2_K \tilde{M}_B^2 \| \bW_{t-1} - \bW_{\text{GD},\infty} \|_F^2
    \end{align*}
    For the third term, define $\bbE_{|\calF_{t-1}}[\cdot] := \bbE [ \cdot | \calF_{t-1}] $ where $\calF_{t-1}$ is the filtration generated by the randomness in the choices of batches up to the time $t-1$:
    
    \begin{align*}
        \bbE_{|\calF_{t-1}} &\biggl[\frac{h'(\|\nabla \calL_t^{(B)}(\bW_{t-1})\|_F)}{\|\nabla \calL_t^{(B)}(\bW_{t-1})\|_F} \\ \cdot & \tr \Bigl( (\bW_{t-1} - \bW_{\text{GD},\infty})^T \nabla \calL_t^{(B)}(\bW_{t-1})\Bigr) \biggr]\\
        &\ge m_K \bbE_{|\calF_{t-1}} \biggl[ \tr \Bigl( (\bW_{t-1} - \bW_{\text{GD},\infty})^T \nabla \calL_t^{(B)}(\bW_{t-1})\Bigr) \biggr]
    \end{align*}
    The inequality follows by the first part of Theorem \ref{thm: conv_stoc} and the assumption on $h(\cdot)$.
    Note that the inequality follows because $\tr \Bigl( (\bW_{t-1} - \bW_{\text{GD},\infty})^T \nabla \calL_t^{(B)}(\bW_{t-1})\Bigr) \ge 0$ by the convexity of $\calL_t^{(B)}$. Observe that
    \begin{align*}
        \bbE_{|\calF_{t-1}} \biggl[ \tr \Bigl( (\bW_{t-1} - \bW_{\text{GD},\infty})^T \nabla \calL_t^{(B)}(\bW_{t-1})\Bigr) \biggr] \\ 
        = \tr \Bigl( (\bW_{t-1} - \bW_{\text{GD},\infty})^T \nabla \calL(\bW_{t-1})\Bigr) 
    \end{align*}
    Then we use the $\mu$-strong convexity of $\ell_{j,j'}$ similar to \cite{akhtiamov2026implicitbiasconvergencematrix}. Let
    \begin{align*}
        \bL_{t-1} = \Bigl[\frac{1}{n}\sum_{j=1}^{n} \ell'_{j,j'} \left(\bx_{j}^T \bW_{t-1}^{(j')}-y_{j j'}\right) \Bigr]_{j'=1}^k
    \end{align*}
    
    Since $\nabla \calL_t^{(B)}(\bW_\ast) = 0$ then by the strong convexity of the functions $\ell_{j,j'}$, we have
    \begin{align*}
         \tr \Bigl( (\bW_{t-1} - \bW_{\text{GD},\infty})^T& \nabla \calL(\bW_{t-1})\Bigr) \\
         &= \tr \Bigl( (\bX\bW_{t-1} - \bX\bW_\ast)^T  \bL_{t-1}\Bigr)\\ 
         &\ge \frac{\mu}{n} \sum_{j'=1}^k \|\bX \bW^{(j')}_{t-1} - \bX \bW^{(j')}_{\ast} \|_F^2 
    \end{align*}
    Then using the fact that $\bW_{t-1} - \bW_{\text{GD},\infty} \in \calS$
    \begin{align*}
         \tr \Bigl( (\bW_{t-1} - \bW_{\text{GD},\infty})^T& \nabla \calL(\bW_{t-1})\Bigr)\ge \tilde{\mu} \|\bW_{t-1} - \bW_{\text{GD},\infty}\|_F^2
    \end{align*}
    Combining these bounds,
    \begin{align*}
        \bbE_{|\calF_{t-1}}  \|\bW_{t} - \bW_{\text{GD},\infty}\|^2_F \\
        \le  \|\bW_{t-1} - \bW_{\text{GD},\infty}\|^2_F \Bigl(1 + \eta^2 L_K^2 \tilde{M}_B^2 - 2 \eta  m_K \tilde{\mu} \Bigr)
    \end{align*}
    This subsequently yields \eqref{ineq: lin_conv}.
\end{proof}


\subsection{Auxiliary Results}

We will now prove Proposition \ref{prop: fund_id}.
\begin{proof} Take an arbitrary $\bW \in \bbR^{d \times k}$. Using the three-point identity of Bregman divergence (Lemma 3.1 of \cite{chen1993convergence}) : 
\begin{align*}
    &\tilde{D}_{\calL}(\bW,\bW_t) +\tilde{D}_{\calL}(\bW_t,\bW_{t-1}) -\tilde{D}_{\calL}(\bW,\bW_{t-1}) \\
    &= \tr{(\left(\bW_{t-1} - \bW_t)^T(\nabla \calL(\bW) -\nabla \calL(\bW_t))\right)}
\end{align*}
Substituting the update rule in \eqref{alg: precond}, we arrive at:
\begin{align}\label{eq: aux_help}
    \tilde{D}_{\calL}&(\bW,\bW_t) +\tilde{D}_{\calL}(\bW_t,\bW_{t-1}) -\tilde{D}_{\calL}(\bW,\bW_{t-1}) \nonumber \\
    &= \eta\tr{\left((\nabla K(\nabla \calL(\bW_{t-1}))^T(\nabla \calL(\bW) -\nabla \calL(\bW_t))\right)}
\end{align}

Lemma 1 of \cite{akhtiamov2026implicitbiasconvergencematrix} applied with $f = K$ gives:
\begin{align}\label{eq: loss_interm}
      \eta D_{K}&(\nabla \calL(\bW),\nabla \calL(\bW_t)) + \eta D_{K}(\nabla \calL(\bW_t),\nabla \calL(\bW_{t-1})) \nonumber\\
      &- \eta D_{K}(\nabla \calL(\bW),\nabla \calL(\bW_{t-1}))   \nonumber \\ 
     & = \eta\tr\Bigl((\nabla K (\nabla \calL(\bW_{t-1})) - \nabla K(\nabla \calL(\bW_t)))^T \nonumber\\
     &(\nabla \calL(\bW) -\nabla \calL(\bW_t))\Bigr)
\end{align}
Subtracting \eqref{eq: loss_interm} from \eqref{eq: aux_help} yields:
\begin{align*}
     &\tilde{D}_{\calL}(\bW,\bW_t) - \eta D_{K}(\nabla \calL(\bW), \nabla \calL(\bW_t)) \\
     &+\tilde{D}_{\calL}(\bW_t,\bW_{t-1})  - \eta D_{K}(\nabla \calL(\bW_t),\nabla \calL(\bW_{t-1}))  \\
     &=  \tilde{D}_{\calL}(\bW,\bW_{t-1}) - \eta D_{K}(\nabla \calL(\bW), \nabla \calL(\bW_{t-1}))\\
     &+\eta\tr{\left((\nabla K(\nabla \calL(\bW_{t}))^T(\nabla \calL(\bW) -\nabla \calL(\bW_t))\right)}
\end{align*}
Expanding $\eta D_{K}(\nabla \calL(\bW),\nabla \calL(\bW_t))$ by definition, and cancelling the common term with $\eta\tr{\left((\nabla K(\nabla \calL(\bW_{t}))^T(\nabla \calL(\bW) -\nabla \calL(\bW_t))\right)}$, we obtain:
\begin{align*}
     &\tilde{D}_{\calL}(\bW,\bW_t) + \eta K(\nabla \calL(\bW_{t})) +\tilde{D}_{\calL}(\bW_t,\bW_{t-1}) \\ &- \eta D_{K}(\nabla \calL(\bW_t),\nabla \calL(\bW_{t-1}))  \\
     &=  \tilde{D}_{\calL}(\bW,\bW_{t-1}) - \eta D_{K}(\nabla \calL(\bW),\nabla \calL(\bW_{t-1}))
\end{align*}
In other words,
\begin{align}\label{eq: fund_id_app}
     \tilde{D}_{\calL}&(\bW,\bW_{t-1}) =\tilde{D}_{\calL}(\bW,\bW_t) \nonumber \\
     &+ \eta D_{K}(\nabla \calL(\bW),\nabla \calL(\bW_{t-1}) ) \nonumber \\
     &+ \eta K(\nabla \calL(\bW_{t})) - \eta K(\nabla \calL(\bW)) \nonumber \nonumber \\
     &+\tilde{D}_{\calL}(\bW_t,\bW_{t-1}) - \eta D_{K}(\nabla \calL(\bW_t),\nabla \calL(\bW_{t-1}))  
\end{align}
This concludes the proof of Proposition \ref{prop: fund_id}. 


\end{proof}

\begin{lemma} \label{lem: nablaL_bnd}
    The following holds
    \begin{enumerate}
        \item If $K - \alpha \calL^{\ast}$ is convex, then 
        \begin{align*}
            \|\nabla \calL(\bW_t)\|_F &\le \sqrt{2\tilde{M} \calL(\bW_0)}   (1 - \alpha \eta)^{t/2}  \\
            &\le \tilde{M} \|\bW_0 -\bW_{\text{GD},\infty} \|_F  (1 - \alpha \eta)^{t/2}
        \end{align*}  
        \item For the gradient descent iterations, we have
        $$\|\nabla\calL( \bW_{\text{GD},t}) \|^2_F \le 2\tilde{M}^2 \|\bW_0 -\bW_{\text{GD},\infty} \|_{F}^2  \left(1-\eta \tilde{\mu}\right)^t $$
    \end{enumerate}
\end{lemma}
\begin{proof}
    Since $K - \alpha \calL^{\ast}$ is convex on $\calS$, we have for $\bW \in \calM$:
    \begin{align*}
        D_{K}(\nabla(\bW),\nabla(\bW_{t-1})) - \alpha \tilde{D}_{\calL}(\bW,\bW_{t-1}) \ge 0
    \end{align*}
    Lemma \ref{lem: conv_rateK} yields an exponential convergence rate for the loss:
    \begin{align*}
        \calL(\bW_t)\le (1 - \alpha \eta)^i \calL(\bW_0)
    \end{align*}
    By Lemma \ref{lem: gen_dec_lem}:
    \begin{align*}
        \|\nabla \calL(\bW_t)\|_F^2 \le 2\tilde{M}\cdot \calL(\bW_t)
    \end{align*}
    Combining with the $\tilde{M}$-smoothness of $\calL$ yields 
    \begin{align*}
        \|\nabla \calL(\bW_t)\|_F \le \sqrt{2\tilde{M} \calL(\bW_0)}   (1 - \alpha \eta)^{i/2} \\
        \le  \tilde{M} \|\bW_0 -\bW_{\text{GD},\infty} \|_{F}  (1 - \alpha \eta)^{i/2} 
    \end{align*}
    Using the $\tilde{M}$-smoothness of $\calL$ together with the results from \cite{akhtiamov2026implicitbiasconvergencematrix}:
    \begin{align*}
        \|\nabla\calL( \bW_{\text{GD},t}) \|^2_{F} \le \tilde{M}^2 \|\bW_{\text{GD},\infty} - \bW_{\text{GD},t}\|_{F}^2 \\
        \le 2\tilde{M}^2 \|\bW_0 -\bW_{\text{GD},\infty} \|_{F}^2 \left(1-\eta \tilde{\mu}\right)^t  
    \end{align*}
\end{proof}

\begin{lemma} \label{lem: breg_k_ell}
Assume that $\eta>0$ is sufficiently small so that $\calL^* - \eta K$ is convex. Then the following holds:
\begin{align*}
   \tilde{D}_{\calL}(\bA,\bB) - \eta D_{K}(\bA,\bB) \ge 0
\end{align*}
\end{lemma}

\begin{proof}
Note that \begin{align*}
     &\tilde{D}_{\calL}(\bA,\bB) - \eta D_{K}(\bA,\bB) = (\calL^* - \eta K)(\nabla \calL(\bA)) \\
     &- (\calL^* - \eta K)(\nabla \calL(\bB))  - \bg^T\left(\nabla \calL(\bA) - \nabla \calL(\bB)\right) \ge 0
\end{align*} where $\bg = \bB - \eta \nabla K(\nabla \calL(\bB))$. The last inequality holds by the additivity of subgradients as $\bB \in \partial \calL^*(\nabla \calL(\bB))$ according to Theorem 23.5 of \cite{rockafellar2015convex}. Therefore, $\bg \in \partial (\calL^* - \eta K)(\nabla \calL(\bB))$.
\end{proof}
\begin{lemma}\label{lem: conv_rateK}
    If for $\bW \in \calM$, $$D_{K}(\nabla \calL(\bW), \nabla \calL(\bW_{t-1})) - \alpha \tilde{D}_{\calL}(\bW,\bW_{t-1}) \ge 0$$ then 
    \begin{align*}
        \tilde{D}_{\calL}(\bW,\bW_t) \le (1-\alpha\eta) \tilde{D}_{\calL}(\bW,\bW_{t-1})
    \end{align*}
\end{lemma}
\begin{proof}
    Proposition \ref{prop: fund_id} yields
    \begin{align*}
        &\tilde{D}_{\calL}(\bW,\bW_{t-1}) - \tilde{D}_{\calL}(\bW,\bW_t) = \eta K(\nabla \calL(\bW_{t}))   \\
        &+\tilde{D}_{\calL}(\bW_t,\bW_{t-1}) - \eta D_{K}(\nabla \calL(\bW_t),\nabla \calL(\bW_{t-1})) \\
        &+ \eta D_{K}(\nabla \calL(\bW),\nabla \calL(\bW_{t-1})) \\
        &\stackrel{(a)}{\ge} \eta K(\nabla \calL(\bW_{t})) + \eta D_{K}(\nabla \calL(\bW),\nabla \calL(\bW_{t-1})) \\
        & \stackrel{(b)}{\ge} \eta D_{K}(\nabla \calL(\bW),\nabla \calL(\bW_{t-1}))
    \end{align*}
    Here, in $(a)$ we have used the assumption on $\tilde{D}_{\calL} - \eta D_{K}$ being nonnegative. $(b)$ follows from the nonnegativity of $K(\cdot)$. Then, by the assumption,
    \begin{align*}
        \tilde{D}_{\calL}(\bW,\bW_{t-1}) - \tilde{D}_{\calL}(\bW,\bW_t) \ge \alpha \eta \tilde{D}_{\calL}(\bW, \bW_{t-1})
    \end{align*}
    This proves the statement.
\end{proof}

\begin{lemma}\label{lem: gen_dec_lem}
    For any convex $L$-smooth function $f: \bbR^{d \times k} \rightarrow \bbR$, if $ 0 <\eta_L < \frac{2}{L}$, then 
    \begin{align*}
         \Bigl(\eta_L - \frac{L\eta_L^2}{2}\Bigr) \|\nabla f(\bx)\|_F^2 \le f(\bx) - \min_{\bx'\in \bbR^{d\times k}} f(\bx')
    \end{align*}
    Consequently,
    \begin{align*}
         \|\nabla f(\bx)\|_F^2 &\le 2L \Bigl(f(\bx) - \min_{\bx'\in \bbR^{d\times k}} f(\bx')\Bigr) \\
         & \le 2L \tr \Bigl(\nabla f(\bx)^T (\bx - \bx^\ast)\Bigr)
    \end{align*}
    where $f(\bx^\ast) =  \min_{\bx' \in \bbR^{d\times k}} f(\bx')$.
\end{lemma}
\begin{proof}
    For any $\bx, \by \in \bbR^{d \times k}$ by the definition of smoothness:
    \begin{align*}
        f(\by) \le f(\bx) + \tr \Bigl(\nabla f(\bx)^T(\by - \bx) \Bigr) + \frac{L}{2} \|\by-\bx\|_F^2
    \end{align*}
    Setting $\by = \bx - \eta_L \nabla f(\bx)$, then
    \begin{align*}
        f\Bigl(\bx - \eta_L \nabla f(\bx)\Bigr) \le f(\bx) +\Bigl(\frac{L\eta_L^2}{2}- \eta_L\Bigr) \|\nabla f(\bx)\|_F^2
    \end{align*}
    Since $f\Bigl(\bx - \eta_L \nabla f(\bx)\Bigr) \ge \min_{\bx'} f(\bx')$, we arrive at
    \begin{align*}
        \Bigl(\eta_L - \frac{L\eta_L^2}{2}\Bigr) \|\nabla f(\bx)\|_F^2 &\le f(\bx) - \min_{\bx' \in \bbR^{d\times k}} f(\bx') \\ 
        &\le 2L \tr \Bigl(\nabla f(\bx)^T (\bx - \bx^\ast)\Bigr)
    \end{align*}
    Where $f(\bx^\ast) =  \min_{\bx' \in \bbR^{d\times k}} f(\bx')$.
\end{proof}

\bibliographystyle{IEEEtran}
\bibliography{sample}

\end{document}